\documentclass[runningheads]{llncs}
\usepackage{graphicx}
\usepackage[misc]{ifsym} %% envelope sign for corresponding author
\usepackage{subcaption} %side by side
\usepackage{pgfplots}
\pgfplotsset{compat=newest}
\usepackage{paralist} % inline item
\usepackage{amsthm}
\usepackage{amsmath}
\usepackage{amssymb}
\usepackage{relsize}
\usepackage{tikz} % draw lines
\usepackage{varwidth}
\usetikzlibrary{arrows.meta,arrows}
\usetikzlibrary{shapes, positioning}

\theoremstyle{plain}
\newtheorem{thm}{Theorem} % reset theorem numbering for each chapter
\theoremstyle{definition}
\newtheorem{defn}{Definition} % definition numbers are independent from theorem numbers
\usepackage{url}
\usepackage{multirow}
\urldef{\mailsa}\path|author1@someinst.something|
\urldef{\mailsb}\path|author2@someotherinst.something|

\usepackage{slashbox}

\usepackage{multicol}
\usepackage{algorithm} 
\usepackage[ruled, vlined, linesnumbered,algo2e]{algorithm2e}
\usepackage[noend]{algpseudocode}
\usepackage{etoolbox}

\makeatletter
\newcommand*{\algrule}[1][\algorithmicindent]{%
  \makebox[#1][l]{%
    \hspace*{.2em}% <------------- This is where the rule starts from
    \vrule height .75\baselineskip depth .25\baselineskip
  }
}

\newcount\ALG@printindent@tempcnta
\def\ALG@printindent{%
    \ifnum \theALG@nested>0% is there anything to print
    \ifx\ALG@text\ALG@x@notext% is this an end group without any text?
    \else
    \unskip
    \ALG@printindent@tempcnta=1
    \loop
    \algrule[\csname ALG@ind@\the\ALG@printindent@tempcnta\endcsname]%
    \advance \ALG@printindent@tempcnta 1
    \ifnum \ALG@printindent@tempcnta<\numexpr\theALG@nested+1\relax
    \repeat
    \fi
    \fi
}
\patchcmd{\ALG@doentity}{\noindent\hskip\ALG@tlm}{\ALG@printindent}{}{\errmessage{failed to patch}}
\patchcmd{\ALG@doentity}{\item[]\nointerlineskip}{}{}{} % no spurious vertical space
\makeatother

\newcommand{\RNum}[1]{\uppercase\expandafter{\romannumeral #1\relax}} % Roman number

\usepackage{listings}
\usepackage[toc,page]{appendix}
\usepackage[utf8]{inputenc}
\usepackage{graphicx}
\graphicspath{ {figures/} }
\usepackage{array}
\usepackage{hyperref}
\usepackage[nottoc,notlot,notlof]{tocbibind}

\usepackage{chngcntr}
\counterwithin{table}{section}% to start numbering scheme with chapter for example Table 3.1 ,used in book
\counterwithin{figure}{section}% to start numbering scheme with chapter for example Fig 3.1 ,used in book
\let\emptyset\varnothing

\usepackage{etoolbox}
\patchcmd{\paragraph}{\itshape}{\bfseries\boldmath}{}{}
\newcommand{\keywords}[1]{\par\addvspace\baselineskip
\noindent\keywordname\enspace\ignorespaces#1}

\begin{document}

\title{High-Utility Interval-Based Sequences}
\author{S. Mohammad Mirbagheri\textsuperscript{(\Letter)} \and Howard J. Hamilton}
%
%\authorrunning{S. Mohammad Mirbagheri et al.} % abbreviated author list (for running head)
%\tocauthor{S.Mohammad Mirbagheri, Howard Hamilton}
%
\institute{Department of Computer Science, University of Regina, Regina, Canada\\
\email {\{mirbaghs,Howard.Hamilton\}}@uregina.ca}
%\\
%\email{Howard.Hamilton@uregina.ca}

\maketitle              % typeset the title of the contribution

\begin{abstract}
Sequential pattern mining is an interesting research area with broad range of applications.
Most prior research on sequential pattern mining has considered point-based data where events occur instantaneously. However, in many application domains, events persist over intervals of time of varying lengths. Furthermore, traditional frameworks for sequential pattern mining assume all events have the same weight or utility. This simplifying assumption neglects the opportunity to find informative patterns in terms of utilities, such as cost. To address these issues, we incorporate the concept of utility into interval-based sequences and define a framework to mine high utility patterns in interval-based sequences i.e., patterns whose utility meets or exceeds a minimum threshold. In the proposed framework, the utility of events is considered while assuming multiple events can occur coincidentally and persist over varying periods of time. An algorithm named \textit{High Utility Interval-based Pattern Miner} (\textit{HUIPMiner}) is proposed and applied to real datasets. To achieve an efficient solution, HUIPMiner is augmented with a pruning strategy. Experimental results show that HUIPMiner is an effective solution to the problem of mining high utility interval-based sequences.
\end{abstract}
\section{Introduction}
\textit{Sequential pattern mining} aims to find patterns from data recorded sequentially along with their time of occurrence. Depending on the application scenario, symbolic sequential data is categorized as either \textit{ point-based} or \textit{interval-based}. Point-based data reflect scenarios in which events happen instantaneously or events are considered to have equal time intervals. Duration has no impact on extracting patterns for this type. Interval-based data reflect scenarios where events have unequal time intervals; here, duration plays an important role.

%Most prior research on sequential data mining has considered point-based data. SPADE \cite{zaki2001spade}, FP-growth \cite{FPGrowth2000}, and PrefixSpan \cite{han2001prefixspan} are well-known algorithms for mining frequent sequential patterns from this type of data. 
%Point-based sequential algorithms may perform poorly when applied to interval-based data, because they ignore the duration of the events and the temporal relations between events. To overcome this problem, interval-based sequential data mining has emerged.

In many application domains, such as medicine \cite{patelHepatit,2019medical}, sensor technology \cite{morchenSensor}, sign language \cite{signLanguage}, and motion capture \cite{motion2016}, events persist over intervals of time of varying lengths, which results in complicated temporal relations among events. Thirteen possible temporal relations between a pair of event intervals were nicely categorized by Allen \cite{allen1983maintaining}. 
Some studies have been devoted to mining frequent sequential patterns from interval-based data and describing the temporal relations among the event intervals. 
%Kim and Fu introduced a hierarchical representation that can be used to describe the temporal relations among events \cite{kam2000discovering}. They recursively defined a temporal pattern as either an atomic pattern or a composite pattern. Then, they proposed an Apriori-like algorithm to find these patterns by generating candidates based on a minimum support and a window size (maximum length of time interval). Unforunately, the patterns represented in this way are ambiguous; a particular temporal relation among events may not be represented by a unique pattern and also a pattern produced in this manner may be mapped to more than one temporal relation. Therefore, 
Wu and Chen \cite{wu2007mining} presented a nonambiguous representation of temporal data utilizing the beginning and ending time points of the events. By adapting the PrefixSpan \cite{han2001prefixspan}, they proposed the TPrefixSpan algorithm to mine frequent temporal sequences. 
%The concept of coincidence was proposed by M{\"o}rchen and Ultsch \cite{TSKR}, where a representation called TSKR was introduced.
Chen et al. \cite{chen2010efficient} proposed the coincidence representation to simplify the processing of complex relations among event intervals. They also proposed an algorithm named CTMiner to discover frequent time-interval based patterns in large databases. 
%Later, they proposed a modified approach \cite{chen2015mining} to discover probabilistic sequential patterns that describe correlations among intervals and to indicate the occurrence probability, which enables users to predict future activities.

The aforementioned work has focused on representations of temporal data and discovering frequent temporal patterns. However, \textit{frequent pattern mining} (FPM) may not be the right solution to problems where the weight of patterns may be the major factor of interest and the frequency of patterns may be a minor factor. The weight of a pattern can be interpreted differently depending on the problem or scenario. For example, it may represent the profit or the cost that a business experiences when
a particular pattern occurs. Some patterns of interest may have high weights but low frequencies. Thus, FPM may miss patterns that are infrequent but valuable. FPM may also extract too many frequent patterns that are low in weight. To address these problems, \textit{high utility pattern mining} (HUPM) has emerged as an alternative to FPM. The goal of HUPM is to extract patterns from a dataset with utility no less than a user-specified minimum utility threshold. 

Tackling the HUPM problem requires facing more challenges than FPM. The major FPM algorithms rely on the downward closure property (also known as the Apriori Property) \cite{srikant1996mining} to perform efficiently. This property, which is utilized by most pruning strategies, states that if a pattern is frequent then all of its sub-patterns are frequent and if a pattern is infrequent all of its super-patterns are infrequent. However, this property does not hold in utility mining because the utilities of patterns are neither monotone nor anti-monotone \cite{Hamilton2006}. As a result, the existing optimization approaches for FPM are not applicapable to HUPM. To cope with this challenge, previous studies introduced several domain-dependent weighted downward closure properties, including the transaction-weighted downward closure property (TDCP) \cite{ahmed2010} for itemset pattern mining, the sequence-weighted downward closure property (SDCP) \cite{uspan} for sequential pattern mining, and the episode-weighted downward closure property (EDCP) \cite{wu2013,Episode2019} for episode pattern mining.
  
Most prior studies on HUPM have been devoted to transactional data rather than sequential data. 
However, such studies do not address the problem of HUPM in interval-based sequences, which covers a wide range of applications mentioned above. Interval-based applications can be better described when the concept of utility is employed. 
For example, interval-based sequences commonly occur in businesses where different services or packages, which persist over time, are offered to customers. Providing informative patterns to policy makers is an essential task, especially in a competitive marketplace. 
Neglecting the fact that these services or packages have various utilities (or weights) results in misleading information.   
For instance, HUPM can be beneficial to telecommunication companies or insurance companies which sell products that last over varying periods of time at various costs. 

To the best of our knowledge, Huang et al. \cite{huang2019} recently made the first attempt to mine interval-based sequence data for patterns based on utility. They suggested a method to discover the top-K high utility patterns (HUPs). Their approach consists of two main parts. It first discovers a set of frequent patterns, and then it extracts the top-K HUPs from the set.
This indirect approach suffers from a major drawback. The set of frequent patterns may not contain all HUPs. Hence, the approach may miss some high utility but infrequent patterns and consequently, it may select low-ranked HUPs as the top-K HUPs.
        
For the above reasons, we formalize the problem of the mining of high utility interval-based patterns (HUIPs) from sequences
and propose a new framework to solve this problem.
The major contributions of this work are as follows: 
\begin{inparaenum}[1)]
\item We propose the \textit{coincidence eventset representation} (\textit{CER}) to represent interval-based events.
%along with their durations, complex temporal relations among such events, and sequences of such events; 
\item We incorporate both internal and external utilities into interval-based sequences
and propose an algorithm called \textit{HUIPMiner} to mine all high utility sequential patterns from interval-based sequences;
%We then formally define the problem of high utility interval-based sequential pattern mining, with the flexibility to choose a desired maximum length and size of the patterns;
\item We introduce the \textit{L-sequence-weighted downward closure property} (\textit{LDCP}), which is used in our pruning strategy and utilize LDCP in HUIPMiner to reduce the search space and identify high utility sequential patterns efficiently; and 
\item We report on experiments that show the proposed framework and algorithm are able to discover all high utility patterns from interval-based data even with a low minimum utility threshold.
\end{inparaenum}

The rest of the paper is organized as follows. Section 2 provides background and preliminaries. It then proposes a framework of interval-based sequence utility and finally it formulates the problem of mining high utility interval-based sequential patterns. Section 3 presents the details of the HUIPMiner algorithm and the pruning strategy. Experimental results on real datasets and evaluation are given in Section 4. Section 5 presents conclusions and future work.

\section{Problem Statement}
%In this section, we introduce definitions and preliminaries realted to   and describe the problem statement formally.%
%We adapt some definitions given in the earlier research \cite{signLanguage,chen2010efficient,uspan,wu2013} and describe the problem statement formally.
%This section introduces the preliminaries related to interval-based sequences and high utility pattern mining and then describe the problem statement formally.
%\begin{defn} %(\textit{Event interval})
Let $\sum=\{A,B,...\}$ denote a finite alphabet. A triple $e=(l,b,f)$, where $l \in \sum$ is the event label, $b \in \mathbb{N} $ is the beginning time, and $f \in \mathbb{N}$ is the finishing time $(b<f)$, is called an \textit{event-interval}.
% Where no ambiguity results, we refer to an event interval $(l,b,f )$ simply by its label $l$. We also use $e.x$ to denote element $x$ of event interval $e$.
%\end{defn}
%\begin{defn} %(\textit{E-sequence})
An event-interval sequence or \textit{E-sequence} $s=\langle e_1,e_2, ...,e_n\rangle $ is a list of $n$ event intervals ordered based on beginning time in ascending order. If event-intervals have equal beginning times, then they are ordered lexicographically by their labels. The size of E-sequence $s$, denoted as $|s|=n$, is the number of event-intervals in $s$.
%\end{defn}
%\begin{defn} %(\textit{E-sequence database})
A database $D$ that consists of set of tuples $\left\langle sid, s \right\rangle $, where $ sid $ is a unique identifier of $s$, is called an \textit{E-sequence database}.
Table \ref{DB} depicts an E-sequence database consisting of four E-sequences with identifiers 1 to 4.
%\end{defn}
\renewcommand{\arraystretch}{1.15} % Increasing the space between two rows
\begin{table*}[t] % the original was ht 
\centering
\caption{Example of an E-sequence database}
\label{DB}
 %\setlength\extrarowheight{-7pt} % reduce padding spaces horizontally
%\scalebox{1}{ % resize
\begin{tabular}{|c|p{1.3cm}|p{1.9cm}|p{1.59cm}|l|}
\hline
\textbf{sid} & \textbf{Event Label} &  \textbf{Beginning Time} & \textbf{Finishing Time} &   \multicolumn{1}{>{\centering\arraybackslash}m{6cm}|}{\textbf{Event Sequence}}  \\ \hline
\multirow{4}{*}{1} & $A$ & 8& 16& \multirow{4}{*}{ \begin{tikzpicture} 
\draw (-2,0) -- node[above] {$A$} ++ (8/3,0) (4/3,-0.5)-- node[above] {$B$} ++(3/3,0) (3.,-0.75)-- node[above] {$C$} ++(4/3,0) (3.3,-1.25)-- node[above] {$E$} ++(2/3,0);
 \end{tikzpicture}} \\ \cline{2-4}
        & $B$ & 18& 21& \\ \cline{2-4}
        & $C$ & 24& 28& \\ \cline{2-4}
        & $E$ & 25& 27& \\ \hline
\multirow{4}{*}{2} & $A$ & 1& 5& \multirow{4}{*}{ \begin{tikzpicture} 
\draw (-2,0) -- node[above] {$A$} ++ (4/3,0) (2/3,-0.3)-- node[above] {$C$} ++(6/3,0) (1,-0.80)-- node[above] {$E$} ++(3/3,0) (1,-1.2)-- node[above] {$F$} ++(3/3,0);
 \end{tikzpicture} } \\ \cline{2-4}
        & $C$ & 8& 14& \\ \cline{2-4}
        & $E$ & 9& 12& \\ \cline{2-4}
        & $F$ & 9& 12& \\ \hline
\multirow{4}{*}{3} & $B$ & 6& 12& \multirow{4}{*}{ \begin{tikzpicture} 
\draw (-2,0) -- node[above] {$B$} ++ (6/3,0) (-1.5,-0.5)-- node[above] {$A$} ++(7/3,0) (.83,-0.75)-- node[above] {$C$} ++(6/3,0) (1.5,-1.25)-- node[above] {$E$} ++(2/3,0);
 \end{tikzpicture} } \\ \cline{2-4}
        & $A$ & 7& 14& \\ \cline{2-4}
        & $C$ & 14& 20& \\ \cline{2-4}
        & $E$ & 16& 18& \\ \hline
\multirow{5}{*}{4} & $B$ & 2& 7& \multirow{5}{*}{ \begin{tikzpicture} 
\draw (-2,0) -- node[above] {$B$} ++ (5/3,0) (-1,-0.4)-- node[above] {$A$} ++(5/3,0)
(-1,-0.9)-- node[above] {$D$} ++(7/3,0) (2.2,-1.2)-- node[above] {$C$} ++(6/3,0) (2.9,-1.7)-- node[above] {$E$} ++(2/3,0);
 \end{tikzpicture}} \\ \cline{2-4}
        & $A$ & 5& 10& \\ \cline{2-4}
        & $D$ & 5& 12& \\ \cline{2-4}
        & $C$ & 16& 22& \\ \cline{2-4}
        & $E$ & 18& 20& \\ \hline 
	\end{tabular}
%	}
\end{table*}
\begin{defn} %(\textit{E-sequence sliced time})
Given an E-sequence $s = \langle (l_1, b_1, f_1),(l_2, b_2, f_2), ...,(l_n, b_n, f_n)\rangle$, the multiset $T =\{b_1, f_1, b_2, f_2, ..., b_n, f_n\}$ consists of all time points corresponding to sequence $s$. If we sort $T$ in ascending order and eliminate redundant elements, we can derive a sequence $T_s =\langle t_1, t_2, ..., t_m \rangle$, where $t_k \in T , t_k < t_{k+1}$. $T_s$ is called the \textit{E-sequence unique time points} of $s$. We denote the number of elements in $T_s$ by $|T_s|$, that is, $|T_s| = m$.
\end{defn}
\begin{defn} %(\textit{Coincidence})
\label{coincidenceDEF}
Let $s=\langle (l_1, b_1, f_1), ...,(l_j, b_j, f_j), ...,(l_n, b_n, f_n)\rangle$ be an E-sequence. A function $\Phi_s: \mathbb{N} \times \mathbb{N} \rightarrow 2^{\sum} $ is defined as:
\begin{equation}
\Phi_s(t_p,t_q)=  \{ l_j \ | \ (l_j, b_j, f_j)  \in s \ \wedge \ ( b_j \leq t_p) \wedge (t_q \leq f_j)   \}
\end{equation}
where $1\leq j \leq n$ and $t_p< t_q$.
 Given an E-sequence $s$ with corresponding E-sequence unique time points $T_s =\langle t_1, t_2, ..., t_m \rangle$, a \textit{coincidence} $c_k$ is defined as $\Phi_s(t_k,t_{k+1})$ where $t_k, t_{k+1} \in T_s$, $1 \leq k \leq m-1 $, are two consecutive time points. The duration $\lambda_k$ of coincidence $c_k$ is $t_{k+1}-t_k$. 
The size of a coincidence is the number of event labels in the coincidence. 
\end{defn}
For example, the E-sequence unique time points of $s_2$ in Table \ref{DB} is $T_{s_2}= \{1,5,8,9,12,14\}$. Coincidence $c_4=\Phi_{s_2}(9,12)= \{C,E,F\}$, $\lambda_4=3$ and $|c_4|=3$.
\begin{defn} %(\textit{Coincidence label sequence})
 %Given an E-sequence $s$, and the corresponding unique time points $T_s =\langle t_1, t_2, ..., t_m \rangle$, 
 A coincidence label sequence, or \textit{L-sequence} $L= \langle c_1c_2...c_{g} \rangle$ is an ordered list of $g$ coincidences.
% where each $c_k$, $(1 \leq k \leq g)$ is a coincidence of E-sequence $s$.
An L-sequence is called a \textit{K}-L-sequence, iff there are exactly $K$ coincidences in the L-sequence. We define the size of an L-sequence, denoted $Z$, to be the maximum size of any coincidences in the L-sequence.

\end{defn}
For example, $\langle \{B\} \{A,B\} \{A\} \rangle $ is a 3-L-sequence because it has 3 coincidences and its size is 2 because the maximum size of the coincidences in it is $max\{1,2,1\}=2$.
\subsection{The Coincidence Eventset Representation (CER)}
%As mentioned earlier,
The representations proposed in previous studies, such as \cite{wu2007mining,chen2010efficient}, do not store the durations of intervals. These approaches transform each event interval into a point-based representation encompassing only temporal relations. Although these formats are described as unambiguous, they actually leave an ambiguity with respect to duration. It is true that the temporal relations among intervals can be mapped one-to-one to the temporal sequence by these representations, but the duration for which these relations persist is ignored. Consequently, it is impossible to reverse the process and reconstruct the original E-sequence if we receive one of these representation. In this section, we address this limitation by incorporating the duration of intervals into a new representation called the \textit{coincidence eventset representation} (CER). 
\begin{defn} %(\textit{Coincidence eventset and eventset sequence})
 Given a coincidence $c_k$ in E-sequence $s$, a coincidence eventset, or \textit{C-eventset}, is denoted $\sigma_k$ and defined as an ordered pair consisting of the coincidence $c_k$ and the corresponding coincidence duration $\lambda_k$, i.e.:
 \begin{equation}
\sigma_k=(c_k, \lambda_k)   
\end{equation} 
%=(\{ l \ | \ l \in c_k  \} , t_{k+1}-t_k )
For brevity, the braces are omitted if $c_k$ in C-eventset $\sigma_k$ has only one event label, which we refer as a \textit{C-event}.
 A coincidence eventset sequence, or \textit{C-sequence}, is an ordered list of C-eventsets, which is defined as $C = \langle \sigma_1 ... \sigma_{m-1} \rangle$, where $m=|T_s|$. 
%\end{defn}
%\begin{defn} %(\textit{C-sequence database})
A \textit{C-sequence database} $\delta$ consists of a set of tuples $\left\langle sid, C \right\rangle $, where $ sid $ is a unique identifier of $C$.
\end{defn} 
For example, the E-sequences in the database shown in Table \ref{DB} can be represented by the CER to give the C-sequences shown in Table \ref{CoincidenceDB}. We denote the $sid=1$ C-sequence as $C_{\mathrm{s}_1}$; other C-sequences are numbered accordingly.
The ``$\emptyset$" symbol is used to distinguish disjoint event intervals. A ``$\emptyset$" indicates a gap between two event intervals, whereas the lack of a ``$\emptyset$" indicates that the two event intervals are adjacent.
It can be seen that CER incorporates the durations of the event intervals into the representation.
\begin{table*}[t]% ht
\centering
\caption{ C-sequence database corresponding to the E-sequences in Table \ref{DB}}
\label{CoincidenceDB}
\begin{tabular}{|c|c|}
\hline
sid & C-sequence   \\ \hline
 1 &  $\langle(A,8) (\emptyset,2)(B,3)(\emptyset,3)(C,1)(\{C,E\},2) (C,1) \rangle$        \\ \hline
 2 &  $\langle (A,4) (\emptyset,3)(C,1)(\{C,E,F\},3) (C,2) \rangle$         \\ \hline
 3 &  $\langle(B,1) (\{A,B\},5)(A,2)(C,2)(\{C,E\},2)(C,2) \rangle$         \\ \hline
 4 &$\langle(B,3)(\{A,B,D\},2)(\{A,D\},3)(D,2)(\emptyset,4)(C,2)(\{C,E\},2)(C,2) \rangle$\\ \hline
\end{tabular}
\end{table*}
\begin{defn}%(\textit{C-subsequence})
Given two C-eventsets $\sigma_a=( c_a , \lambda_a)$
and  $\sigma_b=(c_b , \lambda_b)$, $\sigma_b$ \textit{contains} $\sigma_a$, which is denoted $\sigma_a \subseteq \sigma_b$, iff $ c_a \subseteq c_b \wedge \ \lambda_a=\lambda_b$.
Given two C-sequences $C=\langle \sigma_1 \sigma_2 ... \sigma_{n} \rangle$ and $C'=\langle \sigma_1' \sigma_2' ... \sigma_{n'}' \rangle$, we say $C$ is a \textit{C-subsequence} of $C'$, denoted $C \subseteq C'$, iff there exist integers $1 \leq j_1 \leq j_2 \leq ... \leq j_n \leq n' $ such that $\sigma{_k} \subseteq \sigma_{j_k}'$ for $1 \leq k \leq n$.
%\end{defn}
%\begin{defn}(\textit{Matching})
Given a C-sequence $C=\langle \sigma_1 \sigma_2 ... \sigma_{n} \rangle= \langle (c_1, \lambda_1)(c_2,\lambda_2)...(c_n, \lambda_n) \rangle$ and an L-sequence $L= \langle c_1'c_2'...c_{m}' \rangle$, $C$ \textit{matches} $L$, denoted as $C \sim L$, iff $n = m$ and $c_k=c_k'$ for $1 \leq k \leq n$.
\end{defn}
For example, $\langle (A,2)\rangle$, $\langle (A,2)(A,3) \rangle$, and $\langle (\{A,B,D\},2) \rangle$, are C-subsequences of C-sequence $C_{\mathrm{s}_4}$, while $ \langle (\{A,F\},2) \rangle$ and $ \langle (A,2)(D,5) \rangle$ are not.
It is possible that multiple C-subsequences of a C-sequence match a given L-sequence. For example, if we want to find all C-subsequences of $C_{\mathrm{s}_4}$ in Table \ref{CoincidenceDB} that match the L-sequence $\langle A \rangle$, we obtain $\langle (A,2) \rangle$ in the second C-eventset and $\langle (A,3) \rangle$ in the third C-eventset.
\subsection{Utility}
Let each event label $l \in \sum$, be associated with a value, called the \textit{external utility}, which is denoted as $p(l)$, such that $p: \sum \rightarrow \mathbb{R}_{\geq 0}  $. The external utility of an event label may correspond to any value of interest, such as the unit profit or cost, that is associated with the event label. The values shown in Table \ref{External} are used in the following examples as the external utilities associated with the C-sequence database shown in Table \ref{CoincidenceDB}.
\begin{table} % ht
\centering
\caption{ External utilities associated with the event labels }
\label{External}
\begin{tabular}{|c|c|c|c|c|c|c|c|}
\hline
 Event label & A &B&C&D&E&F& $\emptyset$         \\ \hline
 External utility & 1&2&1&3&2&1 &0       \\ \hline
\end{tabular}
\end{table}
%For example, considering the external utilities in Table \ref{External}, the utility of C-event $(B,3)$ is $\mathrm{u}(B,3)=2 \times 3=6$.
%For example, utility of C-event $(\{A,D\},3)$ is ${\mathrm{u_e}(\{A,D\},3)=3 \times (1+3)=12}$.
%\begin{defn}(\textit{C-sequence database utility})

Let the utility of a C-event $(l,\lambda)$ be $\mathrm{u}(l,\lambda)= p(l) \times \lambda$.
The utility of a C-eventset $\sigma= (\{l_1,l_2, ..., l_n\},\lambda)$ is defined as:
$
\mathrm{u_e}(\sigma)= \sum_{i=1}^{n} \mathrm{u}(l_i,\lambda)
$. The utility of a C-sequence $C=\langle \sigma_1 \sigma_2 ... \sigma_{m} \rangle$ is defined as:
$\mathrm{u_s}(C)= \sum_{i=1}^{m} \mathrm{u_e}(\sigma_i)$.
Therefore, the utility of the C-sequence database $\delta =\{ \langle sid_1, C_{s_1} \rangle, \langle sid_2, C_{s_2} \rangle, ..., \langle sid_r, C_{s_r} \rangle\} $ is defined as:
$
\mathrm{u_d}(\delta)= \sum_{i=1}^{r} \mathrm{u_s}(C_{s_{i}})
$.
For example, the utility of C-sequence $C_{s_3}=\langle(B,1) (\{A,B\},5)(A,2)(C,2)(\{C,E\} ,2)(C,2) \rangle$ is $\mathrm{u_s}(C_{s_3})= 1 \times 2+ 5 \times (1+2) +2 \times 1+ 2 \times 1 +2 \times (1+ 2)+2 \times 1 = 29 $, 
and the utility of the C-sequence database $\delta$ in Table \ref{CoincidenceDB} is $\mathrm{u_d}(\delta)=\mathrm{u_s}(C_{s_1})+\mathrm{u_s}(C_{s_2})+\mathrm{u_s}(C_{s_3})+\mathrm{u_s}(C_{s_4})=22+19+29+46=116$.
%\end{defn}
%
%\begin{defn}(\textit{Maximum C-eventset utility in a C-sequence }) It will not be used probably!
%\begin{equation}
%\mathrm{u_{max_z}}(C)=max\{  \mathrm{u_e}(\sigma) \ | \ \sigma \in C \}
%\end{equation}
%\end{defn}
%For example, the maximum C-eventset utility of C-sequence \\$C_{s_3}=\langle(B,1) (\{A,B\},5)(A,2)(C,2)(\{C,E\},2)(C,2) \rangle$ is $\mathrm{u_{max_z}}(C_{s_3})=15$ 
\begin{defn}%(\textit{Maximum utility of k C-eventsets in a C-sequence}) 
The \textit{maximum utility of $k$ C-eventsets in a C-sequence} is defined as:
$
\mathrm{u_{max_k}}(C,k)=  max\{ \mathrm{u_s}(C ') \ | \ C' \subseteq C \ \wedge \ |C'| \leq k \ \}
$. Note: In the name of the $\mathrm{u_{max_k}}$ function, the ``k'' is part of the name rather than a parameter.
\end{defn}
For example, the maximum utility of 2 C-eventsets in C-sequence $C_{s_3}=\langle(B,1) (\{A,B\},5)(A,2)(C,2)(\{C,E\},2)(C,2) \rangle$ is $\mathrm{u_{max_k}}(C_{s_3},2)=\mathrm{u_s}(\langle (\{A,B\},5)\\ (\{C,E\},2) \rangle)=15+6=21$.
\begin{defn}%(\textit{L-sequence utility set})
Given a C-sequence database $\delta$ and an L-sequence $L= \langle c_1c_2...c_n \rangle$, the utility of $L$ in C-sequence $C=\langle \sigma_1 \sigma_2 ... \sigma_{m} \rangle \in \delta$ is defined as a \textit{utility set}:
\begin{equation}
\mathrm{u_l}(L,C)= \bigcup_{C' \sim L \wedge C' \subseteq C} \mathrm{u_s}(C')
\end{equation}
The utility of $L$ in $\delta$ is also a utility set:
\begin{equation}
\mathrm{u_l}(L)= \bigcup_{C \in \delta} \mathrm{u_l}(L,C)
\end{equation}
\end{defn}
For example, consider L-sequence $L=\langle \{B\}\{A\} \rangle$. The utility of $L$ in $C_{s_3}$ shown in Table \ref{CoincidenceDB} is $\mathrm{u_l}(L,C_{s_3})=\{\mathrm{u_s}(\langle (B,1)(A,5) \rangle),\mathrm{u_s}(\langle (B,1)(A,2)\rangle), \mathrm{u_s}(\langle (B,\\ 5)(A,2) \rangle) \}=\{7,4,12\}$, and thus the utility of $L$ in $\delta$ is $\mathrm{u_l}(L)=\{\mathrm{u_l}(L,C_{s_3}), \\ \mathrm{u_l}(L,C_{s_4})\}=\{\{7,4,12\},\{8,9,7\}\}$. 
From this example, one can see that an L-sequence may have multiple utility values associated with it, unlike a sequence in frequent sequential pattern mining. %This idea is explored further in the next section.   
\subsection{High Utility Interval-based Pattern Mining}
\begin{defn}%(\textit{Maximum L-sequence utility})
The \textit{maximum utility} of an L-sequence $L$ in C-sequence database $\delta$ is defined as $\mathrm{u_{max}}(L)$:
\begin{equation}
\mathrm{u_{max}}(L)= \sum_{C \in \delta} \mathrm{max} (\mathrm{u_{l}}(L,C))
\end{equation}
\end{defn} 
For example, the maximum utility of an L-sequence $L=\langle \{B\}\{A\} \rangle$ in C-sequence database $\delta$ shown in Table \ref{CoincidenceDB} is $\mathrm{u_{max}}(L)=0+0+12+9=21 $.

\begin{defn}%(\textit{High utility interval-based pattern})
An L-sequence $L$ is a \textit{high utility interval-based pattern} iff its maximum utility is no less than a user-specified minimum utility threshold $\xi$. Formally:
$
\mathrm{u_{max}}(L) \geq \xi \iff L \text{ is a high utility interval-based pattern.}
$
\end{defn}
\paragraph{Problem \RNum{1}:}
Given a user-specified minimum utility threshold $\xi$, an E-sequence database $D$, and  external utilities for event labels, the problem of high utility interval-based mining is to discover all L-sequences such that their utilities are at least $\xi$. 
By specifying the maximum length and size of the L-sequence, Problem \RNum{1} can be specialized to give \textbf{Problem \RNum{2}}, which is to discover all L-sequences with lengths and sizes of at most $K$ and $Z$, respectively, such that their utilities are at least $\xi$.
\iffalse  
\paragraph{Problem Statement \RNum{2}.}
Given a user-specified minimum utility threshold $\xi$, an E-sequence database $D$, external utilities for event labels, a maximum L-sequence length $K$, and a maximum L-sequence size $Z$, the problem of high utility interval-based mining is to discover all L-sequences with lengths and sizes of at most $K$ and $Z$, respectively, such that their utilities are at least $\xi$.
\fi
\section{The HUIPMiner Algorithm}
In this section, we propose the \textit{HUIPMiner} algorithm to mine high utility interval-based patterns. HUIPMiner is composed of two phases in which each iteration generates a special type of candidates of a certain length. %We call the first phase the \textit{coincident} phase and the second phase the \textit{serial} phase.
We also obtain the L-sequence-weighted downward closure property (LDCP) (Theorem \ref{theorem1}), which is similar to the sequence-weighted downward closure property (SDCP) \cite{uspan}. LDCP is utilized in the proposed pruning strategy to avoid generating unpromising L-sequence candidates.
LDCP has an advantage over SDCP since it reduces the size of the search space by using a tighter upper bound, which we present in Definition \ref{Upper}.  
%
%\begin{defn}(\textit{L-sequence-weighted utilization of an L-sequence w.r.t a maximum size $z$}: $LWU_\mathrm{z}$) Please ignore this!
%\begin{equation}
%LWU_\mathrm{z}(L,z) = \sum_{C' \sim L \wedge C' \subseteq C \wedge C  \in \delta } \mathrm{u_{maxZ}}(C,z) 
%\end{equation} 
%\end{defn}

\begin{defn}(\textit{$\mathrm{LWU}_k$}) 
\label{Upper}
The L-sequence-weighted utilization of an L-sequence w.r.t. a maximum length $k$ is defined as:
\begin{equation}
\mathrm{LWU}_k(L) = \sum_{C' \sim L \wedge C' \subseteq C \wedge C  \in \delta } \mathrm{u_{max_k}}(C,k) 
\end{equation} 
\end{defn}
For example, the L-sequence-weighted utilization of $L=\langle \{B\}\{A\} \rangle$ w.r.t. the maximum length $k=2$ in the C-sequence database shown in Table \ref{CoincidenceDB} is $\mathrm{LWU}_2(\langle \{B\}\{A\} \rangle)=0+0+21+24=45$.
%\iffalse
\begin{thm}[L-sequence-weighted downward closure property]
\label{theorem1}
Given a C-sequence database $\delta$ and two L-sequences $L$ and $L'$, where $L \subseteq L'$ and $|L'| \leq k$, then
\begin{equation}
\mathrm{LWU}_k(L') \leq \mathrm{LWU}_k(L)
\end{equation} 
\end{thm}
%\iffalse
\begin{proof}
Let $\alpha$ and $\beta$ be two C-subsequences that match the L-sequences $L$ and $L'$, respectively. Since $L \subseteq L'$, then $\alpha\subseteq \beta$. 
Let $Q' \in \delta$ be the set of all C-sequences containing $\beta$ and $Q \in \delta$ be the set of all C-sequences containing $\alpha$.
Since $\alpha \subseteq \beta$, then $Q$ must be a superset of $Q'$, that is, $Q \supseteq Q'$. Therefore, we infer 
\begin{equation}
\sum_{\substack{\beta \sim L'  \wedge \beta \subseteq C'  \wedge C'  \in Q' }} \mathrm{u_{max_k}}(C',k) \leq
\sum_{\substack{\alpha \sim L  \wedge \alpha \subseteq C  \wedge C  \in Q }} \mathrm{u_{max_k}}(C,k)
\end{equation}
and equivalently we derive $\mathrm{LWU}_k(L') \leq \mathrm{LWU}_k(L)$.
\end{proof}
%\fi
Algorithm \ref{HUIPMinerAlg} shows the main procedure of the HUIPMiner algorithm. The inputs are: (1) a C-sequence database $\delta$, (2) a minimum utility threshold $\xi$, (3) a maximum pattern length $K \geq 1$, and (4) a maximum pattern size $Z \geq 1$. The output includes all high utility interval-based patterns. The algorithm has two phases, a \textit{coincident phase} to obtain high utility coincidence patterns (L-sequences with lengths equal to 1) and a \textit{serial phase} to obtain high utility serial patterns (L-sequences with lengths greater than 1).  
\subsection{The Coincident Phase}
The coincident phase, which is the first phase of HUIPMiner (Lines 1-13), generates coincidence candidates by concatenating event labels. 
\begin{defn}%(\textit{Coincident concatenation})
\label{Cconcat}
Let $c=\{l_1,l_2,...,l_n\}$ and $c'=\{l_1',l_2',...,l_m'\}$ be two coincidences. The \textit{coincident concatenation} of $c$ and $c'$ is the ordinal sum of the coincidences and is defined as coincident-concat$(c,c')=(c \cup c', \leq)=c \oplus c'$. 
\end{defn}
For example, coincident-concat$(\{A,B\},\{A,C\})=\{A,B,C\}$.

In the first round of this phase, all event labels are considered as coincidence candidates with a size of 1 (Line 1). Then, the algorithm searches each C-sequence to find matches to these candidates. Next, it calculates the maximum utility $\mathrm{u_{max}}$ and L-sequence-weighted utilization $\mathrm{LWU}_k$ of each candidate. If $\mathrm{u_{max}}$ for a candidate is no less than the given threshold $\xi$, then the candidate is classified as a high utility coincident pattern and placed in set HUCP. 
For example, suppose we want to find all HUIPs of Table \ref{CoincidenceDB} when the threshold is 14, the maximum size of a coincidence $Z$ is 2, and the maximum length of an L-sequence $K$ is 2. For simplicity, suppose all event labels have equal external utilities of 1.    
Table \ref{HUIPM1} shows the coincidence candidates of size 1 and their maximum utilities and L-sequence-weighted utilizations, which are denoted $\mathrm{LWU}_2$. At the end of the first round, \{A\} is the only candidate that is added into HUCP because $\mathrm{u_{max}(\langle \{A\} \rangle)} \geq 14$.
 
\begin{table} % ht
\centering
\caption{ HUIPMiner example - coincidence phase }
\label{HUIPM1}
\begin{tabular}{|c|c|c|c|c|c|c|}
\hline
Candidate 			& \{A\} & \{B\} & \{C\} & \{D\} &\{E\} & \{F\} \\ \hline
$\mathrm{u_{max}}$ & 20 & 11& 9 & 3 & 9 & 3  \\ \hline
$\mathrm{LWU}_2$ & 51 & 38& 51 & 12 & 51 & 13 \\ \hline
 
\end{tabular}
\end{table}
Before the next round is started, coincidence candidates of size 2 are generated. In order to avoid generating too many candidates, we present a pruning strategy, which is based on the following definition.
\begin{defn}%(\textit{Promising coincidence candidate})
 A coincidence candidate $c$ is \textit{promising} iff $\mathrm{LWU}_k(c) \geq \xi$. Otherwise it is unpromising.   
\end{defn}
\noindent \textit{Property.}  Let $a$ be an unpromising coincidence candidate and $a'$ be a coincidence. Any superset produced by coincident-concat$(a,a')$ is of low utility.   
\\ \textit{Rationale.} Property 1 holds by the LDCP property (Theorem \ref{theorem1}).
\\ \textit{Pruning strategy.} Discard all unpromising coincidence candidates. \\
If the $\mathrm{LWU}_k$ value of a candidate is less than $\xi$, the candidate will be discarded since it is unpromising. If the $\mathrm{LWU}_k$ value of a candidate is no less than $\xi$, the candidate is promising and thus it will be added to set $P$, the set of promising candidates for the current run. The HUIPMiner algorithm also extracts the unique elements of the promising candidates (Line 10). Before the algorithm performs the next round, $P$ is added into WUCP, which is the set of all weighted utilization coincident patterns with sizes up to $Z$. WUCP is later used in the serial phase.
In our example, the algorithm prunes (discards) \{D\} and \{F\} in the first round because their $\mathrm{LWU}_2$ values are less than 14. Therefore, \{D\} and \{F\} will not be involved in generating candidates for the second round. 
 \{A\}, \{B\}, \{C\} and \{E\} are identified as promising candidates and added into $P$. Then, coincidence candidates of size 2 are generated for the next round by calling the \textit{Ccandidate} procedure and sending $P$ and the unique elements as input arguments (Definition \ref{Cconcat}). The algorithm repeats this procedure until it reaches $Z$ or no more candidates can be generated. At the end of this phase, the algorithm has found all high utility coincident patterns and stored them in HUCP; it has also found all weighted utilization coincident patterns of maximum size $Z$ such that $\mathrm{LWU}_k$ is no less than $\xi$ and stored them in WUCP. In the serial phase, WUCP is used to find the serial patterns.
%For example, In the second round, the two promising candidates \{A,B\} and \{C,E\} are added into WUCP.
% Next, the serial phase is started by serial concatenation of \{A\} with WUCP members. 
 
\subsection{The Serial Phase}
In the serial phase, the second phase of HUIPMiner (Lines 14-27), serial candidates are generated by concatenating the weighted utilization coincident patterns found in the first phase. 
\begin{defn}%(\textit{Serial concatenation})
\label{Sconcat}
Let $L=\langle c_1,c_2,...,c_n \rangle $ and $L'= \langle c_1',c_2', ...,c_m' \rangle$ be two L-sequences. The \textit{serial concatenation} of $L$ and $L'$ is defined as serial-concat$(L,L') \\ =\langle c_1,c_2,...,c_n,c_1',c_2',..., c_m' \rangle$.
\end{defn} 
 For example, the serial concatenation of two L-sequences  $L=\langle \{A,B\}, \{A,C\} \rangle$ and $L'=\langle \{E\},\{D,C,F\} \rangle$ is $L''=\langle \{A,B\},\{A,C \}, \{E\},\{D,C,F\} \rangle$.

In the first round of this phase, all serial L-sequence candidates of length 2 are generated. For this purpose, each coincident pattern $w$ in WUCP is used to generate serial L-sequence candidates that start with $w$ as the first coincidence of the L-sequence. This is done by calling the \textit{Scandidate} procedure and sending $w$ and WUCP as input arguments (Definition \ref{Sconcat}). Then, the algorithm searches each C-sequence in the C-sequence database to find matches to serial L-sequence candidates. The search for matches in this phase is more challenging than the search in the coincidence phase. It requires that the order of the coincidences also be taken into account. Therefore, it adds more complexity as the length of the L-sequence increases. After matches are found, as in the coincidence phase, the algorithm calculates $\mathrm{u_{max}}$ and $\mathrm{LWU}_k$ of every serial candidate. If $\mathrm{u_{max}}$ for a candidate $l$ is no less than the given threshold $\xi$, then $l$ is classified as a high utility serial pattern (HUSP). If $\mathrm{LWU}_k$ for a serial candidate $l$ is no less than threshold $\xi$, then $l$ is added into the set of promising candidates $P$.
In order to generate longer serial candidates, the algorithm extracts the unique coincidences located at the $k^{th}$ position of the candidate (last coincidence) and stores them in $NewL$.
%Next, serial candidates of length 3 are generated for the next round by calling the \textit{Scandidate} procedure and sending $P$ and $NewL$ as input arguments.
Next, Scandidate procedure generates serial candidates of length 3 for the next round by serially concatenating $P$ and $NewL$. 
The algorithm repeats these steps until it reaches the maximum length of patterns $K$ or no more candidates can be generated. At the end of this phase, the algorithm has found all high utility serial patterns with lengths up to $K$ and stored them in HUSP. After the serial phase ends, the high utility coincident and serial patterns are sent to the output. \\
%%%%%%%%% implemented
\begin{algorithm}[h]
\caption{HUIPMiner: High Utility Interval-based Pattern Miner}\label{HUIPMinerAlg}
  %\begin{multicols}{2}
  \begin{minipage}[t]{6cm}
\null 
  {\setlength{\algoheightrule}{0pt}
\setlength{\algotitleheightrule}{0pt}%
\begin{algorithm2e}[H]
\DontPrintSemicolon
\SetKwFunction{Ccandidate}{Ccandidate}\SetKwFunction{Scandidate}{Scandidate}
 \KwIn{A C-sequence database $\delta$, minimum utility threshold $\xi$, maximum length $K \geq 1$, maximum size $Z \geq 1$ } 
 \KwOut{All high utility interval-based patterns $HUIP$}
 Initialize the set of high utility coincident patterns $HUCP=\emptyset$, the set of weighted utilization coincident patterns $WUCP=\emptyset$, $z=1$, and $C^z=$ all event labels\;
\While{$z \leq Z$ and $C^z \neq \emptyset$} {
 $P=\emptyset$, $NewL=\emptyset$ \;
\For{each candidate $c$ in $C^z$}{
  	 Find $c$ in $\delta$ and Calculate $\mathrm{LWU}_K(c)$ \;
    \If{$\mathrm{u_{max}}(c)\geq \xi$}
    { $HUCP=HUCP \cup c$
    }
     \If{$\mathrm{LWU}_K(c) \geq \xi$}
     { $P=P \cup c$ \;
     $NewL=NewL \cup \{p \ | \ p \in c \}$\;
      }
       }
       $WUCP=WUCP \cup P$\;
       $z=z+1$\;
       $C^{z}=\text{Ccandidate}(P,NewL)$\;
    }
  \end{algorithm2e}
  }
%\columnbreak
\end{minipage}
\begin{minipage}[t]{5.9cm}
\null 
{\setlength{\algoheightrule}{0pt} % thickness of the rules above and below
\setlength{\algotitleheightrule}{0pt}%  % thicknes of the rule below the title
  \begin{algorithm2e}[H]
  \DontPrintSemicolon
  \setcounter{AlgoLine}{13}%% split the first alg and continue from 28 
 Initialize the set of high utility serial patterns $HUSP=\emptyset$ and $k=2$ \;
\For{each weighted utilization pattern $w$ in $WUCP$} {
 $L^{k}=Scandidate(w,WUCP)$ \;
\While{$k \leq K$ and $L^k \neq \emptyset$} {
%    	 $S= PatternFinder(L^k,\xi)$ \;
   $P=\emptyset$, $NewL=\emptyset$ \;
  \For{each candidate $l$ in $L^k$}{
  	 Find $l$ in $\delta$ and Calculate $\mathrm{LWU}_K(l)$ \;
    \If{$\mathrm{u_{max}}(l)\geq \xi$} {
     $HUSP=HUSP \cup l$ \;
    }
     \If{$\mathrm{LWU}_K(l)\geq \xi$}{
     $P=P \cup l$ \;
      $NewL=NewL \cup \{ k^{th}$ coincidence in $l \}$\;
      }
    }
 		 		  
 		   $k=k+1$ \;
 		   $L^{k}=Scandidate(P,NewL)$ \;
}  
}
 $HUIP=HUCP \cup HUSP$\;
\end{algorithm2e}
}
% \end{multicols}
\end{minipage}
\end{algorithm}
%%%%%%%%%%%%%%%%%%%%%%%%%%%%%%%%%
\section{Experiments}
The HUIPMiner algorithm was implemented in C++11 and tested on a desktop computer with a 3.2GHz Intel Core 4 CPU and 32GB memory. 
%in C++11 using the g++ compiler
We used four real-world datasets from various application domains in our experiments to evaluate the performance of HUIPMiner. 
The datasets include three publicly available datasets, namely Blocks \cite{morchenSensor}, Auslan2 \cite{morchenSensor}, ASL-BU \cite{signLanguage}, and a private dataset, called DS, obtained from our industrial partner.
DS includes event labels corresponding to various services offered to customers. An E-sequence in this dataset represents a customer receiving services.
The minimum, maximum and average external utilities associated with the event labels in DS are 10, 28, and 18, respectively. There are no external utilities associated with the public datasets. Therefore, we assume every event label in these datasets have an external utility of 1. The statistics of the datasets are summarized in Table \ref{StatDB}.
%The statistic of external utilities associated with the services are shown in Table \ref{ExtDB}.     
\begin{table*}[ht] % ht
\centering
\caption{ Statistical information about datasets }
\label{StatDB}
\begin{tabular}{|c|c|c|c|c|c|c|c|c|c|c| }
\hline
 Dataset &\multicolumn{1}{p{1.35cm}|}{\centering \# \hspace{0pt} Event Intervals  } &  \multicolumn{1}{p{1.75 cm}|}{\centering \hspace{0pt} \#  E-sequences} & \multicolumn{3}{c|}{E-sequence Size} & \multicolumn{1}{p{1 cm}|}{\centering \# \\ Labels} & \multicolumn{4}{c|}{Interval  Duration}   \\ 
   		 & & &	 min	& max & avg&   &min	& max & avg & stdv   \\ \hline
  Blocks & 1207 & 210 &3 &12 &6 &8 &1&57&17  &12   \\ \hline
   Auslan2 & 2447 & 200 & 9 & 20& 12 &12  &1&30& 20 & 12       \\ \hline
    ASL-BU & 18250 & 874 &3 & 40&17 & 216 & 3 &4468& 594 & 590       \\ \hline
     DS & 71416 & 10017& 4 & 14& 8& 15& 1 &484  &70  & 108     \\ \hline
\end{tabular}
\end{table*}
\subsection{Performance Evaluation}
We evaluate the performance of HUIPMiner on the four datasets in terms of their execution time and the number of extracted high utility patterns, while varying the minimum utility threshold $\xi$ and the maximum length of patterns $K$. These two evaluations are shown on a log-10 scale in Figure \ref{ExThreshold} and Figure \ref{ExLength}, respectively. 
The execution time of HUIPMiner in seconds is shown on the left and the number of patterns discovered by HUIPMiner is presented on the right of the two figures.
%In order to test the proposed algorithm under proper stress, 
The maximum size of patterns $Z$ is set to 5 in all experiments.
% In both Figures, the execution time of HUIPMiner on the four datasets are shown on the left side and the number of extracted patterns corresponding to the datasets are presented on the right side.

Figure \ref{ExThreshold} shows the evaluation of the HUIPMiner on the datasets while varying $\xi$ and keeping $K$ set to 4. 
The algorithm is able to discover a large number of HUIPs in a short time, especially for smaller datasets. For instance, the algorithm can extract more than 4500 HUIPs in about 60 seconds from Blocks under a low minimum utility. It is evident that as $\xi$ increases, the execution time drops exponentially and fewer patterns are discovered. This is especially well supported for larger datasets like ASL-BU and DS. Apart from the way that event intervals are distributed, the large number of event labels in ASL-BU are the major factor that contributes to high computational costs for extracting patterns. Similarly, the large number of E-sequences in DS requires more execution time to extract patterns from this dataset. The results also show that HUIPMiner is effective at finding patterns for small thresholds.
  
Figure \ref{ExLength} shows the evaluation of the HUIPMiner on the four datasets when $K$ is varied between 1 and 4. In these experiments, a small $\xi$ corresponding to each dataset is used to benchmark the algorithm.
As shown in Figure \ref{ExLength}, HUIPMiner discovers a high number of HUIPs from Blocks in a short time when $\xi$ is set to 0.02. The algorithm performs similarly on Auslan2 when $\xi=0.01$. 
When the algorithm is applied to ASL-BU and DS, patterns are discovered at lower speeds than from the two other datasets, 
when the minimum thresholds are set to 0.1 and 0.05, respectively.
As expected, $K$ plays an important role in determining both the execution time of the algorithm and the number of extracted patterns. As $K$ increases, the execution time increases and more patterns are discovered.

In general, the performance of the algorithm depends on the dataset characteristics (mentioned in Table \ref{StatDB}) as well as the parameters used in the experiments ($Z$, $K$, $\xi$). The experiments show that HUIPMiner can successfully extract high utility patterns from datasets with different characteristics under various parameters setups.    
%%%%%%%%%%%%%%%%%%%%%%%%%%%
%%%%%%%%%%%%%%%%%%%%%%
%%%%%% FIGGGGGGGGGGGGGGGGGGGGGG 1111111111111111
    \newenvironment{customlegend}[1][]{%
        \begingroup
        % inits/clears the lists (which might be populated from previous
        % axes):
        \csname pgfplots@init@cleared@structures\endcsname
        \pgfplotsset{#1}%
    }{%
        % draws the legend:
        \csname pgfplots@createlegend\endcsname
        \endgroup
    }%

    % makes \addlegendimage available (typically only available within an
    % axis environment):
    \def\addlegendimage{\csname pgfplots@addlegendimage\endcsname}
%%%%%%%%%%%%%%%%%%%% Threshold
%%%%%%%%%%%%%%%%%%%%% FIGGGGGGGGGGGGGGGGGG1
\begin{figure}[ht]
\pgfplotsset{width=3cm}
%\pgfplotsset{compat=1.16}
\begin{subfigure}{.45 \linewidth}\centering
%%%%%%%%%%%%%%%%%%%%%%%% BLOCKS
\begin{tikzpicture}%[scale=\linewidth/15cm]

\begin{axis}[ tick label style={/pgf/number format/fixed},scale only axis,  xlabel= $\xi$, ylabel= Time (s),ymode=log,log basis y={10}]
    %% BLOCKS Times
	\addplot[color=black,mark=square] coordinates {
    (0.02,62.716) (0.052,35.470) (0.078,25.420)
 (0.10,19.488)(0.13,14.603)
    };
	%%%%%%%%% Auslan TIMEs
    \addplot[color=blue,mark=star] coordinates {
 (0.012,216.171) (0.015,202.625) (0.03,26.441)
 (0.045,21.494)(0.08,3.618)
    };
	%%%%%%%%% ASL-BU TIMEs
	\addplot[color=red,mark=o] coordinates {
 (0.10,13895.1) (0.111,6311.71) (0.123,3269.8)
 (0.134,2018.5)(0.145,1367.5)
     };
	 %%%%%%%%% DS TIMEs
\addplot[color=violet,mark=diamond] coordinates {
    (0.05,5703.29) (0.07,3150.49) (0.09,1990.58)
 (0.11,1224.31)(0.13,827.034)
    }; 
   	\end{axis}
   	\end{tikzpicture}
	\subcaption{Execution time}
	\label{ExeTime}
\end{subfigure}
%%%%%%%%%%%%%%%%%%%%%%%%%%%%
%%%%%%%%%%%%%%%%%%%%%%%%%  Number of patterns
\begin{subfigure}{.45 \linewidth}\centering
\begin{tikzpicture}
\begin{axis}[tick label style={/pgf/number format/fixed},scale only axis, xlabel= $\xi$, ylabel= Number of Patterns,ymode=log,log basis y={10}]
    %%%% Blocks number of patterns
	\addplot[color=black,mark=square] coordinates {
    (0.02,4589) (0.052,1996) (0.078,1171)
 (0.10,585)(0.13,222)
  };   %%%% Auslan number of patterns
   \addplot[color=blue,mark=star] coordinates {
    (0.012,2281)(0.015,1282)(0.03,161)(0.045,50)(0.08,19)
  }; 
  %%%% ASL-BU number of patterns
	\addplot[color=red,mark=o] coordinates {
    (0.10,1835) (0.111,475) (0.123,137)
 (0.134,41)(0.145,7)
  };
  %%%% DS number of patterns
   \addplot[color=violet,mark=diamond] coordinates {
    (0.05,939) (0.07,356) (0.09,178)
 (0.11,86)(0.13,49)
  };
   	\end{axis}
   	\end{tikzpicture}
	\subcaption{Number of patterns}
	\label{NumPattern}
	\end{subfigure}
	\begin{tikzpicture} \centering
        \begin{customlegend}[legend columns=4,legend style={at={(9.5,2.5)},anchor=south,align=center,draw=none,column sep=2ex},legend entries={Blocks ,Auslan ,ASL-Bu ,DS }]
        \addlegendimage{color=black,mark=square}
        \addlegendimage{color=blue,mark=star} 
        \addlegendimage{color=red,mark=o}
        \addlegendimage{color=violet,mark=diamond}   
        \end{customlegend}
        \end{tikzpicture}
	\caption{Execution time and number of patterns for different $\xi$ values  }
	\label{ExThreshold}
\end{figure}
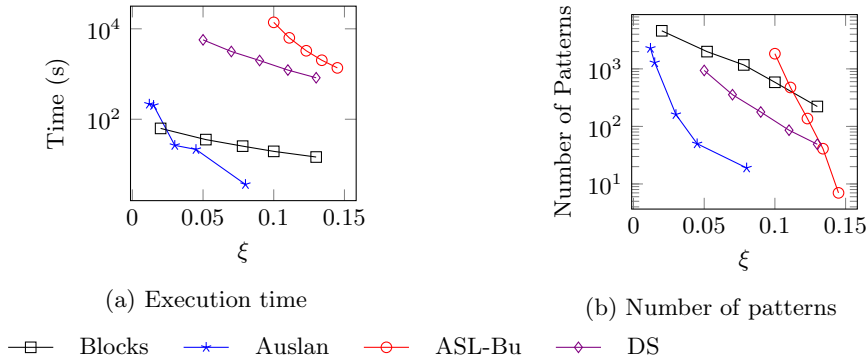
%%%%%%%%%%%%%%%%%%%%%%%%%%%%%%%%%%%%%%%%%%%
%%%%%%%%%%%%%%%%%%%%% FIGGGGGGGGGGGGGGGGGG2
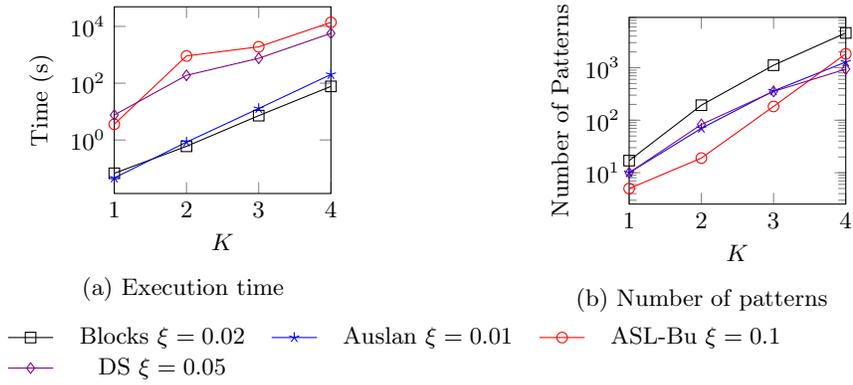
\begin{figure}[ht]
\pgfplotsset{width=2.88cm}
%\pgfplotsset{compat=1.16}
\begin{subfigure}{.4 \linewidth}\centering
%%%%%%%%%%%%%%%%%%%%%%%% BLOCKS
\begin{tikzpicture}%[scale=\linewidth/15cm]

\begin{axis}[ scale only axis, xmin=1,xmax=4, xtick={1,2,3,4}, xlabel= $K$, ylabel= Time (s),ymode=log,log basis y={10}]
    %% BLOCKS Times
	\addplot[color=black,mark=square] coordinates {
    (1,0.067)          
    (2,0.598)(3,7.196)
    (4,77.142)
    };
	%%%%%%%%% Auslan TIMEs
    \addplot[color=blue,mark=star] coordinates {
 (1,0.046) (2,0.840) (3,12.835)
 (4,202.625) 
    };
	%%%%%%%%% ASL-BU TIMEs
	\addplot[color=red,mark=o] coordinates {
 (1,3.58) (2,910.856) (3,1909.52)
 (4,13895.1)
     };
	 %%%%%%%%% DS TIMEs
\addplot[color=violet,mark=diamond] coordinates {
    (1,7.591)          
    (2,188.899)(3,753.369)
    (4,5703.29)
    }; 
   	\end{axis}
   	\end{tikzpicture}
	\subcaption{Execution time}
	\label{ExeTime}
\end{subfigure}
%%%%%%%%%%%%%%%%%%%%%%%%%%%%
%%%%%%%%%%%%%%%%%%%%%%%%%  Number of patterns
\begin{subfigure}{.47 \linewidth}\centering
\begin{tikzpicture}
\begin{axis}[scale only axis, xmin=1,xmax=4, xtick={1,2,3,4}, xlabel= $K$, ylabel= Number of Patterns,ymode=log,log basis y={10}]
    %%%% Blocks number of patterns
	\addplot[color=black,mark=square] coordinates {
    (1,17)
    (2,193)(3,1115)
    (4,4589)
  };   %%%% Auslan number of patterns
   \addplot[color=blue,mark=star] coordinates {
    (1,10)
    (2,70)(3,361)
    (4,1282)
  }; %%%% ASL-BU number of patterns
	\addplot[color=red,mark=o] coordinates {
    (1,5)
    (2,19)(3,183)
    (4,1835)
  };
  %%%% DS number of patterns
   \addplot[color=violet,mark=diamond] coordinates {
    (1,10)
    (2,83)(3,353)
    (4,939)
  };
   	\end{axis}
   	\end{tikzpicture}
	\subcaption{Number of patterns}
	\label{NumPattern}
	\end{subfigure}
	\begin{tikzpicture} \centering
        \begin{customlegend}[legend columns=3,legend style={at={(9.5,2.5)},anchor=south,align=center,draw=none,column sep=2ex},legend entries={Blocks $\xi=0.02$,Auslan $\xi=0.01$,ASL-Bu $\xi=0.1$,DS $\xi=0.05$}]
        \addlegendimage{color=black,mark=square}
        \addlegendimage{color=blue,mark=star} 
        \addlegendimage{color=red,mark=o}
        \addlegendimage{color=violet,mark=diamond}   
        \end{customlegend}
        \end{tikzpicture}
	\caption{Execution time and number of patterns for different $K$ values }
	\label{ExLength}
\end{figure}
%%%%%%%%%%%%%%%%%%%%%%%%%%%%%%%%%%%%%%%%%%%
\subsection{Effect of pruning strategies}
The computational benefits of the proposed pruning strategy is also evaluated. We compare our pruning strategy, which is based on the LDCP property, against a pruning strategy based on the SDCP property and also against the execution of HUIPMiner when no pruning strategy is applied. Figure \ref{Prunning} shows the time for the strategies on Blocks dataset with $\xi=0.02$. The LDCP based pruning strategy is a dominant winner on this dataset in comparison with no pruning. LDCP is also more efficient than SDCP, especially when the maximum length of patterns increases. This result is further supported in Fig \ref{Prunning2} where LDCP is compared against SDCP on the Auslan2 dataset. Similar results were obtained with various values of $\xi$ and on other datasets. 
%%%%%%%%%%%%%%%%%%% BLOCKS PRUNNING
\begin{figure}[ht]
\pgfplotsset{width=3.5cm}
\begin{subfigure}{.5 \linewidth}\centering
\begin{tikzpicture}
\begin{axis}[scale only axis, xmin=1,xmax=4, xtick={1,2,3,4},ytick={0,50,...,350}, ymin=0,ymax=350, xlabel= $K$, ylabel= Time (s),legend pos=north west]
 \addplot[color=red,mark=x,mark size=2.9pt] coordinates {
    (1,0.079)          
    (2,1.331)(3,19.561)
    (4,310.839)
    };
     \addlegendentry{No Pruning} 
	\addplot[color=blue,mark=o] coordinates {
    (1,0.067)          
    (2,0.598)(3,7.196)
    (4,77.142)
    };
\addlegendentry{SDCP}   
       \addplot[color=black,mark=square] coordinates {
    (1,0.056)          
    (2,0.518)(3,6.134)
    (4,62.716)
    };
    \addlegendentry{LDCP}  
   \end{axis}
\end{tikzpicture}
\subcaption{Blocks: $\xi=0.02$}
\label{Prunning}
\end{subfigure}
%%%%%%%%%%%%%%%%%%%%%%%%%%%%%%%%%%%%%%%%%%
%%AUSLAN SDCP VS LDCP 2222222222222
\begin{subfigure}{.5 \linewidth}\centering
\begin{tikzpicture}
\begin{axis}[scale only axis, xmin=1,xmax=5, xtick={1,2,3,4,5}, ymin=0,ymax=80, xlabel=$K$, ylabel= Time (s),legend pos=north west]
	\addplot[color=blue,mark=o	] coordinates {
    (1,0.032)          
    (2,0.155)(3,1.142)
    (4,8.824)(5,69.290)    };
\addlegendentry{SDCP}   
       \addplot[color=black,mark=square] coordinates {
    (1,0.009)          
    (2,0.076)(3, 0.529)
    (4, 3.618)(5, 29.990)
    };
    \addlegendentry{LDCP}  
   \end{axis}
\end{tikzpicture}

\subcaption{Auslan: $\xi=0.08$}
\label{Prunning2}
\end{subfigure}
\caption{Comparison of pruning strategies}
\end{figure}
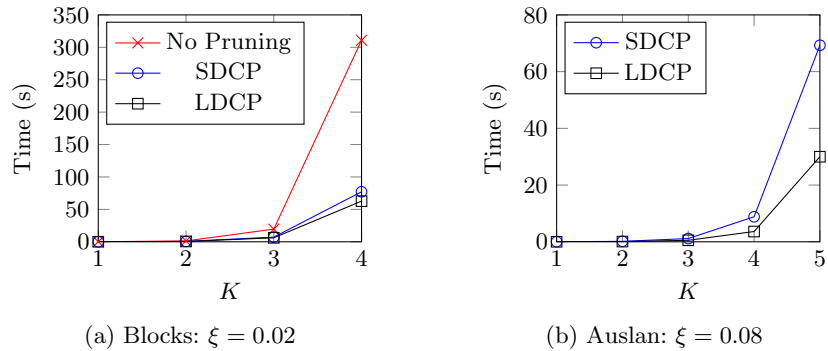
%%%%%%%%%%%%%%%%%%%%%%%%%%%%%%%%%%%%%%%%%%
\section{Conclusions and Future Work}
Mining sequential patterns from interval-based data is more challenging than mining from point-based data due to the existence of complex temporal relations among events. Seeking high utility patterns increases the complexity of the problem because the downward closure property does not hold. In this paper, we proposed the coincidence eventset representation to express temporal relations among events along with the duration of events. This representation simplifies the description of complicated temporal relations without losing information. We incorporated the concept of utility into interval-based data and provided a novel framework for mining high utility interval-based sequential patterns. An effective algorithm named HUIPMiner was proposed to mine patterns. Furthermore, in order to mine the dataset faster, a pruning strategy based on LDCP was proposed to decrease the search space. Experimental evaluations have shown that HUIPMiner is able to identify patterns with low minimum utility.

Utility mining in interval-based sequential data could provide benefits in diverse applications. For instance, more industries could take advantage of the utility concept to model their monetary or non-monetary considerations. In medicine, alternatives for  courses of treatment over a long period may have different utilities. Our approach could be applied to find high utility alternatives from records of many patients with long-lasting diseases. Similarly, managers could utilize the high utility patterns in making decisions about increasing profits based on many sequences of events with durations.   
 %%%%%%%%%%%copy  up to here	
\paragraph*{Acknowledgments.} 
The authors wish to thank Rahim Samei (Technical Manager at ISM Canada) and the anonymous reviewers for the insightful suggestions. This research was supported by funding from ISM Canada and NSERC Canada.
%\newpage
%\clearpage
%%%%%%%%%%%%%%%% END HERE
%\newpage
%\clearpage
%\nobalance
%\bibliographystyle{ACM-Reference-Format}
 \bibliographystyle{splncs} 
% first uncomment , compile, restore
%\bibliography{bibfile} % ShortBibfile
\let\chapter\section\bibliography{bibfile} %No page breaks before bibliography
%\begin{thebibliography}{00}
%\end{thebibliography}
\end{document}